\documentclass[12pt]{article}

\usepackage{fullpage}
\usepackage[nocompress]{cite}
\usepackage{natbib}
\usepackage{subfigure}
\usepackage[colorlinks,linkcolor=red,citecolor = blue]{hyperref}
\usepackage{graphicx}
\usepackage[cmex10]{amsmath}
\usepackage{amssymb}
\usepackage{amsthm}
\usepackage{bbm}
\usepackage{array}
\usepackage{mdwmath}
\usepackage{mdwtab}
\usepackage{fixltx2e}
\usepackage{stfloats}
\usepackage{url}
\usepackage{color}
\usepackage{hyperref}
\usepackage{lettrine}
\usepackage[usenames,dvipsnames]{xcolor}
\usepackage{framed}
\usepackage{cancel}
\colorlet{shadecolor}{black!15}
\usepackage{algorithm, algorithmic}
\usepackage{bbm}
\usepackage{times}
\usepackage[font=small,labelfont=bf]{caption}

\newtheorem{theorem}{Theorem}
\newtheorem{definition}{Definition}
\newtheorem{lemma}{Lemma}
\newtheorem{proposition}{Proposition}
\newcommand{\card}[1]{{\left|#1\right|}}
\newcommand{\ssq}{S$^2$}

\newtheorem{claim}{Claim}
\title{\ssq: An Efficient Graph Based Active Learning Algorithm with Application to Nonparametric Classification}

\author{Gautam Dasarathy%
\footnote{Machine Learning Department, Carnegie Mellon University; gautamd@cs.cmu.edu}
 \and
Robert Nowak%
\footnote{Department of Electrical Engineering, University of Wisconsin, Madison; {rdnowak@wisc.edu}}
\and
Xiaojin Zhu%
\footnote{Department of Computer Sciences, University of Wisconsin, Madison;  jerryzhu@cs.wisc.edu} 
}
\date{}

\begin{document}

\maketitle

\begin{abstract}
This paper investigates the problem of active learning for binary label prediction on a graph.
We introduce a simple and label-efficient algorithm called \ssq for this task. At each step, \ssq selects the vertex to be labeled based on the structure of the graph and all previously gathered labels. Specifically, \ssq queries for the label of the vertex that bisects the {\em shortest shortest} path between any pair of oppositely labeled vertices. We present a theoretical estimate of the number of queries \ssq needs in terms of a novel  parametrization of the complexity of binary functions on graphs. We also present experimental results demonstrating the performance of \ssq on both real and synthetic data. While other graph-based active learning algorithms have shown promise in practice, our algorithm is the first with both good performance and theoretical guarantees. Finally, we demonstrate the implications of the \ssq algorithm to the theory of nonparametric active learning. In particular, we show that \ssq achieves near minimax optimal excess risk for an important class of nonparametric classification problems.
\end{abstract}

\section{Introduction}
\label{sec.intro}
This paper studies the problem of binary label prediction on a graph. We suppose that we are given a graph over a set of vertices, where each vertex is associated with an initially unknown binary label. For instance, the vertices could represent objects from two classes, and the graph could represent the structure of similarity of these objects. The unknown labels then indicate the class that each object belongs to. The goal of the general problem of binary label prediction on a graph is to predict the label of all the vertices given (possibly noisy) labels for a subset of the vertices. Obtaining this initial set of labels could be costly  as it may involve consulting human experts or expensive experiments. It is therefore of considerable interest to minimize the number of vertices whose labels need to be revealed before the algorithm can predict the remaining labels accurately. In this paper, we are especially interested in designing an \emph{active} algorithm for addressing this problem, that is, an algorithm that sequentially and automatically selects the vertices to be labeled based on both the structure of the graph and the previously gathered labels. We will now highlight the main contributions of the paper: 
\begin{itemize}
\item {\bf A new active learning algorithm which we call ${\rm S}^2$.} In essence, \ssq, which stands for \emph{shortest shortest path} operates as follows. Given a graph and a subset of vertices that have been labeled, \ssq  picks the mid-point of the path of least length among all the shortest paths connecting oppositely labeled vertices.  As we will demonstrate in Sec~\ref{sec.analysis}, \ssq automatically adapts itself to a natural notion of the complexity of the cut-set of the labeled graph.
\item {\bf A novel complexity measure}. While prior work on graph label prediction  has focused only on the cut-size (i.e., the number of edges that connect oppositely labeled nodes)  of the labeling, our refined complexity measure (cf. Section~\ref{sec.complexity}) quantifies the difficulty of learning the cut-set. Roughly speaking, it measures how clustered the cut-set is; the more clustered the cut-set, the easier it is to find. This is analogous to the fact that the difficulty of classification problems in standard settings depends both on the \emph{size} and the \emph{complexity} of the Bayes decision boundary. 
\item {\bf A practical algorithm for non-parametric active learning.} Significant progress has been made in terms of characterizing the theoretical advantages of active learning in nonparametric settings (e.g., \cite{castro08,wang11,Minsker12,hanneke11,koltchinskii10}), but most methods are not easy to apply in practice. On the other hand, the algorithm proposed in \cite{zhu03combining}, for example, offers a flexible approach to nonparametric active learning that appears to provide good results in practice. It however does not come with theoretical performance guarantees.  A contribution of our paper is to fill the gap between the practical method of \cite{zhu03combining} and the theoretical work above. We show that \ssq achieves the minimax rate of convergence for classification problems with decision boundaries in the so-called {\em box-counting} class (see \cite{castro08}).  To the best of our knowledge this is the first practical algorithm that is minimax-optimal (up to logarithmic factors) for this class of problems.
\end{itemize}

\subsection{Related Work}
Label prediction on the vertices of a graph is an important and challenging problem. Many practical algorithms have been proposed for this (e.g.,~\cite{dasgupta,zhuGhahLaff03,blumChawla01,blumLafferty04}),  
with numerous applications such as information retrieval \citep{joachim03transductive} and learning intracellular pathways in biological networks \citep{chasmanCraven}, among others. 
Theoretically, however, a complete understanding of this problem has remained elusive. In the active learning version of the problem there even appears to be contradictory results with some supporting the benefit of active learning~\citep{dasgupta,afshani07} while others being pessimistic~\citep{cesaBianchi10}. In this paper, we propose a  new and simple active learning algorithm, \ssq,  and examine its performance. Our theoretical analysis of \ssq, which utilizes a novel parametrization of the complexity of labeling functions with respect to graphs, clearly shows the benefit of adaptively choosing queries on a broad class of problems. The authors in \cite{cesaBianchi10}  remark that ``adaptive vertex selection based on labels'' is not helpful in \emph{worst-case adversarial settings}.  Our results do not contradict this since we are concerned with more realistic and non-adversarial labeling models (which may be deterministic or random).

Adaptive learning algorithms for the label prediction on graphs follow one of two approaches. The algorithm can either (i) choose {all its queries upfront}
according to a fixed design based on the structure of the graph without observing any labels (e.g., in \cite{GuHan12bound, cesaBianchi10}), or 
(ii) pick vertices to query in a sequential manner based on the structure of the graph {\em and previous collected labels} (e.g., \cite{zhu03combining,afshani07})\footnote{This distinction is important since the term ``active learning'' has been used in \cite{cesaBianchi10} and \cite{GuHan12bound} in the sense of (i)}.
The former is closely related to experimental design in statistics, while the latter is adaptive and more akin to active learning; this is the focus of our paper. 

Another important component to this problem is using the graph structure and the labels at a subset of the vertices to predict the (unknown) labels at all other vertices.  This arises in both passive and active learning algorithms. This is a well-studied problem in semi-supervised learning and there exist many good methods \citep{zhuGhahLaff03,zhou04llgc,GuHan12bound,blumChawla01}.  The focus of our paper is the adaptive selection of vertices to label.  Given the subset of labeled vertices it generates, any of the existing methods mentioned above can be used to predict the rest.

The main theoretical results of our paper characterize the sample complexity of learning the cut-set of edges that partition the graph into disjoint components corresponding to the correct underlying labeling of the graph. 
The work most closely related to the focus of this paper is \cite{afshani07}, which studies this problem from the perspective of query complexity \citep{angluin2004queries}. Our results improve upon those in \cite{afshani07}.  The algorithm we propose is able to take advantage of the fact that the cut-set edges are often close together in the graph, and this can greatly reduce the sample complexity in theory and practice.

Finally, our theoretical analysis of the performance of \ssq quantifies the number of labels required to learn the cut-set and hence the correct labeling of the entire graph.  It does not quantify the number of labels needed to achieve achieve a desired {\em nonzero} (Hamming) error level. To the best of our knowledge, there are no algorithms for which there is a sound characterization of the Hamming error. Using results from \cite{guillory09}, \cite{cesaBianchi10} takes a step in this direction but their results are valid only for trees.  For more general graphs, such Hamming error guarantees are unknown. Nonetheless, learning the entire cut-set induced by a labeling guarantees a Hamming error of zero, and so the two goals are intimately related.

\section{Preliminaries and Problem Setup}
\label{sec.preliminaries}
We will write $G = (V,E)$ to denote an undirected graph with vertex set $V$ and edge set $E$. Let $f:V\to \left\{ -1,+1 \right\}$ be a binary function on the vertex set. We call $f(v)$ the \emph{label} of $v\in V$. We will further suppose that we only have access to the labeling function $f$ through a (potentially) noisy oracle. That is, fixing $\gamma\in (0,0.5]$, we will assume that for each vertex $v\in V$, the oracle returns a random variable $\hat{f}(v)$ such that $\hat{f}(v)$ equals $-f(v)$ with probability $\gamma$ and equals $f(v)$ with probability $1 - \gamma$, independently of any other query submitted to this oracle. 
We will refer to such an oracle as a $\gamma-$\emph{noisy} oracle. 
The goal then is to design an algorithm which sequentially selects a multiset\footnote{A multiset is a set that may contain repeated elements.} of vertices $L$ and uses the labels $\left\{ \hat{f}(v), v\in L \right\}$ to accurately learn the true labeling function $f$. Since the algorithm we design will assume nothing about the labeling, it  will be equipped to handle even an adversarial labeling of the graph. Towards this end, if we let $C$ denote the \emph{cut-set} of the labeled graph, i.e., $ C\triangleq \left\{ \left\{ x,y \right\}\in E: f(x)\neq f(y) \right\}$ and let $\partial C$ denote the \emph{boundary} of the cut-set, i.e., $\partial C = \left\{ x\in V: \exists e\in C \mbox{ with }x\in e \right\}$, then our goal will actually be to identify $\partial C$. 

Of course, it is of interest to make $L$ as small as possible. Given $\epsilon\in (0,1)$, the number of vertices that an algorithm requires the oracle to label, so that with probability at least $1 - \epsilon$ the algorithm correctly learns $f$ will be referred to as its \emph{$\epsilon-$query complexity}. We will now show that given an algorithm that performs well with a noiseless oracle, one can design an algorithm that performs well on a $\gamma-$noisy oracle. 

\begin{proposition}
\label{prop.noise}
Suppose $\mathcal{A}$ is an algorithm that has access to $f$ through a noiseless oracle, and suppose that it has a $\epsilon-$query complexity $q$, then for each $\gamma\in (0,0.5)$, there exists an algorithm $\tilde{\mathcal{A}}$ which, using a $\gamma-$noisy  oracle achieves a $2 \epsilon-$query complexity given by $q\times \left\lceil \frac{1}{2(0.5 -\gamma)^2}\log\left( \frac{n}{\epsilon} \right) \right\rceil$.

\end{proposition}

The main idea behind this proposition is that one can build a noise-tolerant version of $\mathcal{A}$ by repeating each query that $\mathcal{A}$ requires many times, and using the majority vote. The proof of Proposition~\ref{prop.noise} is then a straightforward application of Chernoff bounds and we defer it to Appendix~\ref{sec.proofProp1}. 

Therefore, to keep our presentation simple, we will assume in the sequel that our oracle is noiseless. A more nuanced treatment of noisy oracles is an interesting avenue for future work. 

It should also be noted here that the results in this paper can be extended to the multi-class setting, where $f: V\to \left\{ 1,2,\ldots, k \right\}$ by the standard ``one-vs-rest'' heuristic (see e.g., \citet{bishop2006pattern}). However, a thorough investigation of the multiclass classification on graphs is an interesting avenue for future work.

\section{The \ssq algorithm}
\label{sec.algorithm}
The name \ssq signifies the fact that the algorithm bisects the  \emph{\underline{s}hortest \underline{s}hortest-path} connecting oppositely labeled vertices in the graph. As we will see, this allows the algorithm to automatically take advantage of the clusteredness of the cut set. Therefore, \ssq   ``\emph{unzips}'' through a tightly clustered cut set and locates it rapidly (see Figure~1). 
\begin{figure}[t]
\label{fig.S2example}
  \subfigure[Random sampling ends. One shortest shortest path is shown with  thickened edges.]{\includegraphics[width=0.24\textwidth]{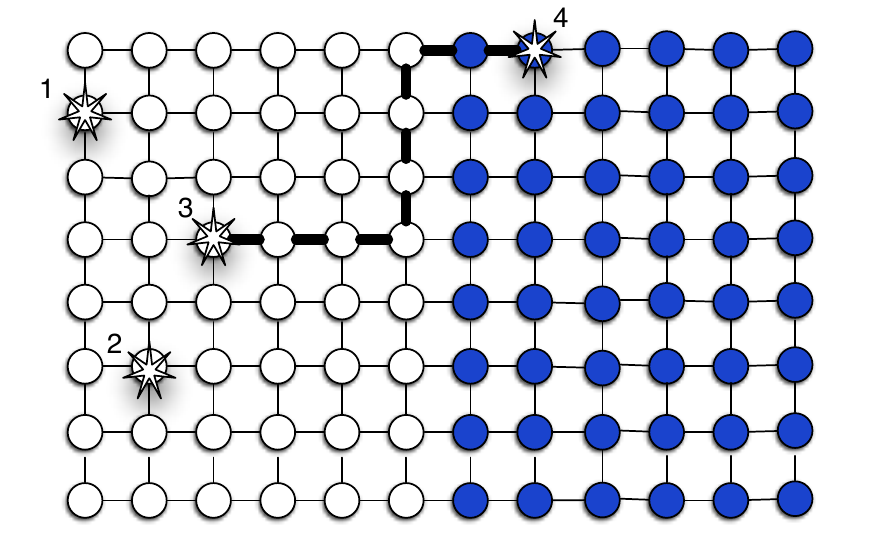}}\hfill
  \subfigure[This path bisected and \ssq finds one cut edge.]{\includegraphics[width=0.24\textwidth]{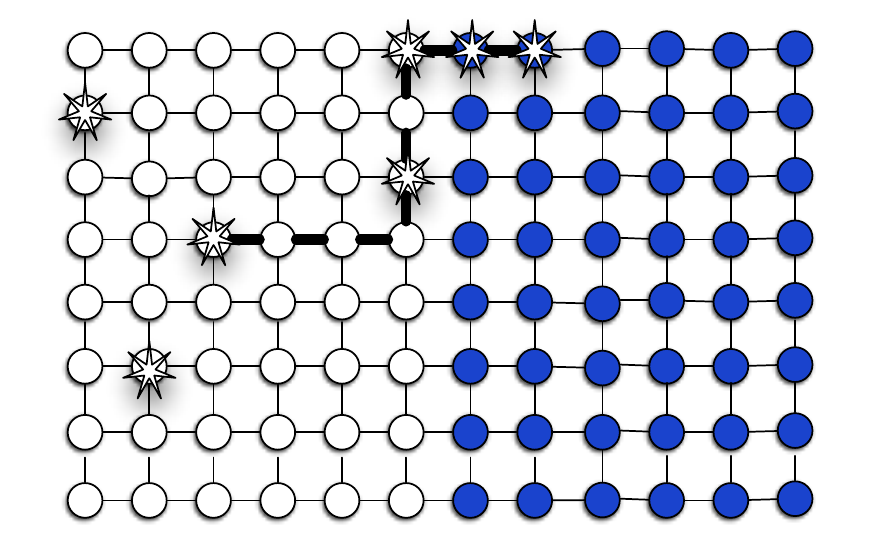}}
\hfill  \subfigure[Next shortest-shortest path is bisected.]{\includegraphics[width=0.24\textwidth]{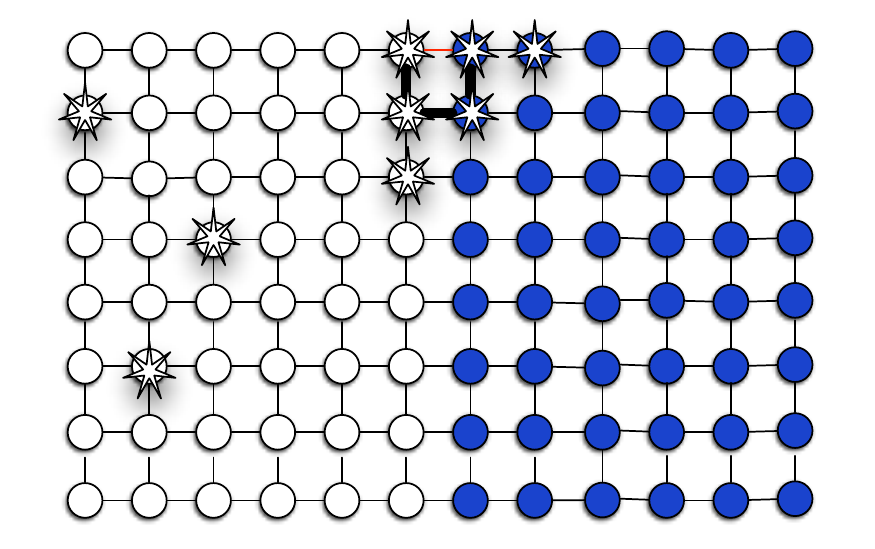}}\hfill
  \subfigure[This continues till \ssq unzips through the cut boundary.]{\includegraphics[scale = 0.42
 ]{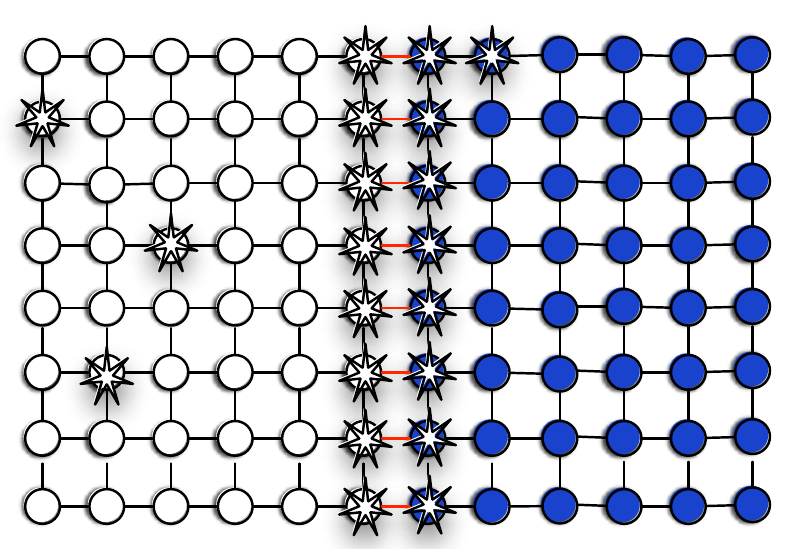}}
\caption{A sample run of the \ssq algorithm on a grid graph. The shaded and unshaded vertices represent two different classes (say $+1$ and $-1$ respectively). See text for explanation.}
\end{figure}

The algorithm works by first constructing a label set $L\subset V$ sequentially and adaptively and then using the labels in $L$ to predict the labels of the vertices in $V\setminus L$. It accepts as input a natural number \textsc{budget}, which is the query budget that we are given to complete the learning task. In our theoretical analysis, we will show how big this budget needs to be in order to perform well on a wide variety of problems. Of course this budget specification merely reflects the completely agnostic nature of the algorithm; we subsequently show (see Section~\ref{sec.stoppingCriterion}) how one might factor in expert knowledge about the problem to create more useful stopping criteria.
\renewcommand{\algorithmicrepeat}{\textbf{do}}
\renewcommand{\algorithmicuntil}{\textbf{while}}
\renewcommand{\algorithmicrequire}{\textbf{Input}}

\begin{algorithm}[t!]
\centering
\captionof{algorithm}{\ssq: Shortest Shortest Path}
\label{alg.main}
\begin{algorithmic}[1]
\REQUIRE{Graph ${G}=(V,E)$, {\textsc{budget}} $ \leq n$}
\STATE $L\leftarrow\emptyset$
\WHILE{$1$}
\STATE $x\leftarrow$ Randomly chosen unlabeled vertex
\REPEAT
\STATE Add $(x,f(x))$ to $L$
\STATE Remove from $G$ all edges whose two ends have different labels. 
\IF{$\card{L}=\mbox{\textsc{budget}}$}
\STATE {{\bf Return} \textsc{labelCompletion}$(G,L)$}
\ENDIF
\UNTIL{{$x\leftarrow\mbox{\textsc{mssp}} (G,L)$ exists}}
\ENDWHILE
\end{algorithmic}
\end{algorithm}

\floatname{algorithm}{Sub-routine}

\begin{algorithm}[t!]
\centering
\captionof{algorithm}{\textsc{mssp} \textcolor{white}{$S^2 g$}}
\label{alg.mssp}

\begin{algorithmic}[1]

\REQUIRE{Graph ${G}=(V,E)$, $L\subseteq V$}

\FOR{each $v_i,v_j\in L$ such that $f(v_i)\neq f(v_j)$}
\STATE $P_{ij}\leftarrow$ shortest path between $v_i$ and $v_j$ in ${G}$
\STATE $\ell_{ij}\leftarrow$ length of $P_{ij}$ ($\infty$ if no path exists)
\ENDFOR

\STATE $(i^\ast,j^\ast)\leftarrow\arg\min_{v_i,v_j\in L: f(v_i)\neq f(v_j)}\ell_{ij}$
\IF{$(i^\ast,j^\ast)$ exists}
\STATE {\bf Return} mid-point of $P_{i^\ast j^\ast}$ (break ties arbitrarily).
\ELSE
\STATE{\bf Return} $\emptyset$
\ENDIF
\end{algorithmic}
\end{algorithm}

\ssq (Algorithm~\ref{alg.main}) accepts a graph ${G}$ and a natural number \textsc{budget}. Step 3 performs a random query. In step 6, the obvious cut edges are identified and removed from the graph. This ensures that once a cut edge is found, we do not expend more queries on it.  In step 7, the algorithm ensures that it has enough budget left to proceed. If not, it stops querying for labels and completes the labeling using the subroutine \textsc{labelCompletion}. In step 10, the mid-point of the shortest shortest path is found using the \textsc{mssp} subroutine (Subroutine~\ref{alg.mssp}) and the algorithm proceeds till there are no more mid-points to find. Then, the algorithm performs another random query and the above procedure is repeated. As discussed in Proposition~\ref{prop.noise}, if in Step 5, we compute $f(x)$ as the majority label among $\mathcal{O}\left(\log(n/\epsilon)\right)$ repeated queries about the label of $x$, then we get a \emph{noise tolerant version} of \ssq. 

We will now describe the sub-routines used by \ssq. Each of these accepts a graph $G$ and a set $L$ of vertices. \textsc{labelCompletion}$(G,L)$ is any off-the-shelf procedure that can complete the labeling of a graph from a subset of labels. This could, for instance, be the graph min-cut procedure~\cite{blumChawla01} or harmonic label propagation~\cite{zhuGhahLaff03}. 
Since, we endeavor to learn the entire cut boundary, we only include this sub-routine for the sake of completeness so that the algorithm can run with any given budget.
Our theoretical results do not depend on this subroutine.
\textsc{mssp}$(G,L)$ is Subroutine~\ref{alg.mssp} and finds the midpoint on the shortest among all the shortest-paths that connect oppositely labeled vertices in $L$.  If none exist, it returns $\emptyset$.

The main idea underlying the \ssq algorithm is the fact that learning the labeling function $f$ amounts to locating all the \emph{cut-edges} in the labeled graph. Conceptually, the algorithm can be thought of as operating in two phases: random sampling and aggressive search. In the random sampling phase, the algorithm queries randomly chosen unlabeled vertices until it finds two vertices with opposite labels. Then our algorithm enters the aggressive search phase. It picks the shortest path between these two points and bisects it. What our algorithm does next sets it apart from prior work such as \citet{afshani07}. It does not run each binary search to the end, but merely keeps bisecting the shortest among all the shortest paths that  connect oppositely labeled vertices observed so far. This endows the algorithm with the ability to ``unzip'' cut  boundaries. Consider a sample run shown in Figure~1. The random sampling phase first picks a set of random vertices till an oppositely labeled pair is observed as in Figure~1(a). The shortest shortest path connecting oppositely labeled nodes is shown here as a thick sequence of edges. Figure~1(b) shows that \ssq now bisects  shortest shortest paths and subsequently finds a cut-edge. The bisection of the next two shortest shortest paths is shows in Figure~1(c) and we see the boundary unzipping feature of the algorithm emerges. Figure~1(d) finally shows the situation after the completion of the algorithm. Notice that an extremely small number of queries are used before \ssq completely discovers the cut boundary.

\subsection{Stopping Criterion}
\label{sec.stoppingCriterion}
Notice \ssq stops only if the budget is exhausted. This stopping criterion was chosen for two main reasons. Firstly, this keeps the presentation simple and reflects the fact that \ssq   \emph{assumes nothing} about the underlying problem. In practice, extra information about the problem can be easily incorporated. For instance, suppose one knows a bound on the cut-set or on the size of similarly labeled connected components, then such criteria can be used in Step 7. Similarly, one can hold out a random subset of observed labels and stop upon achieving low prediction error on this subset. Secondly, in our theoretical analysis, we show that as long as the budget is large enough, then \ssq recovers the cut boundary exactly.
The larger the budget is, the more complex the graphs and labelings \ssq can handle. Therefore, our result can be interpreted as a quantification of the complexity of a natural class of graphs.
In that sense \ssq is not only an algorithm but a tool for exposing the theoretical challenges of label prediction on a graph.

\section{Analysis}
\label{sec.analysis}

Let $C$ and $\partial C$ be as defined in Section~\ref{sec.preliminaries}. 
 As we observed earlier, the problem of learning the labeling function $f$ is equivalent to the problem of locating all the cut edges in the graph. Clearly, if $f$ could be arbitrary, the task of learning $f$ given its value on a subset of its domain is ill-posed. However, we can make this problem interesting by constraining the class of functions that $f$ is allowed to be in. Towards this end, in the next section we will discuss the complexity of a labeling function with respect to the underlying graph. 

\subsection{Complexity of the Labeling Function with respect to the Graph}
\label{sec.complexity}
We will begin by making a simple observation. Given a graph $G$ and a labeling function $f$, notice that $f$ partitions the vertices of the graph into a collection of connected components with identically labeled vertices. We will denote these connected components as $V_1, V_2,\ldots, V_k$, where $k$ represents the number of connected components. 

Then, the above partitioning of the vertex set induces a natural partitioning of the cut set $C = \bigsqcup_{1\leq r<s\leq k}C_{rs}$, where $C_{rs} = \{{x,y}\in C: x\in V_r, y\in V_s\}.$ That is, $C_{rs}$ contains the cut edges whose boundary points are in the components $V_r$ and $V_s$\footnote{$C_{rs}$ could of course be empty for certain pairs $r,s$.}. We denote the number of non-empty subsets $C_{rs}$ by $m$, and we call each non-empty $C_{rs}$ a \emph{cut component}. See Figure~\ref{fig.cutComponent}(a) for an illustration. It shows a graph, its  corresponding labeling (denoted by darker and lighter colored vertices) and the cut set (thickened lines). It also shows the induced vertex components $V_1,V_2,V_3$ and the corresponding cut components $C_{12}$ and $C_{23}$. We will especially be interested in how clustered a particular cut component is. For instance, compare the cut component $C_{12}$ between Figures~\ref{fig.cutComponent}(a) and \ref{fig.cutComponent}(b). Intuitively, it is clear that the former is more clustered than the latter. As one might expect, we show that it is easier to locate well-clustered cut components.

We will now introduce three parameters which, we argue, naturally govern the complexity of a particular problem instance. Our main results will be expressed in terms of these parameters.

{\bf 1. Boundary Size.} The first and the most natural measure of complexity we consider is the size of the boundary of the cut set $\left|\partial C\right|$. It is not hard to see that a problem instance is hard if the boundary size induced by the labeling is large. In fact, $\left|\partial C\right|$ is trivially a lower bound on the total number of queries needed to learn the location of $C$. Conversely, if it is the case that well-connected vertices predominantly have the same label, then $\card{\partial C}$ will most likely be small. Theorem~\ref{thm.main} will show that the number of queries needed by the \ssq algorithm scales approximately linearly with $\left|\partial C\right|$.

{\bf 2. Balancedness.} The next notion of label complexity we consider is the \emph{balancedness} of the labeling. As we discussed above, the labeling function induces a partition on the vertex set $V_1,V_2,\ldots,V_k$. We will define the \emph{balancedness of $f$ with respect to ${G}$} as $\beta\triangleq\min_{1\leq i\leq k}\frac{\card{V_i}}{n}$.

It is not hard to see why balancedness is a {natural way of quantifying problem  complexity}. If there is a very small component, then without prior knowledge, it will be unlikely that we see labeled vertices from this component. We would then have no reason to assign labels to its vertices that are different from its neighbors. Therefore, as we might expect, the larger the value of $\beta$, the easier it is to find a particular labeling (see Lemma~\ref{lem.randomSampling} for more on this).

{\bf {3. Clustering of Cut Components.}} We finally introduce a notion of complexity of the labeling function which, to the best of our knowledge, is novel. As we show, this complexity parameter is key to developing an efficient active learning algorithm for general graphs. 

Let $d_{\rm sp}^G : V\times V \to \mathbb{N}\cup\left\{ 0,\infty \right\}$ denote the shortest path metric with respect to the graph $G$, i.e., $d_{\rm sp}^G(x,y)$ is the the length of the shortest path connecting $x$ and $y$ in $G$ with the convention that the distance is $\infty$ if $x$ and $y$ are disconnected. Using this, we will define a metric $\delta : C\times C \to \mathbb{N}\cup\left\{ 0,\infty \right\}$ on the cut-edges as follows. Let $e_1 = \left\{ x_1,y_1 \right\}, e_2 = \left\{ x_2,y_2 \right\} \in C$ be such that $f(x_1) = f(x_2) = +1$ and $f(y_1) = f(y_2) = -1$. Then $\delta(e_1,e_2) = d_{\rm sp}^{G - C}(x_1,x_2) + d_{\rm sp}^{G-C}(y_1,y_2) + 1$, where $G-C$ is the graph $G$ with all the cut-edges removed. Notice that $\delta$ is a metric on $C$ and that $\delta(e_1,e_2) < \infty$ if and only if $e_1$ and $e_2$ lie in the same cut-component. 

Imagine a scenario where a cut-component $C_{rs}$ is such that each pair $e,e'\in C_{rs}$ satisfy $\delta(e,e')\leq \kappa$. Now, suppose that the end points of one of the cut-edges $e_1 \in C_{rs}$ has been discovered. By assumption, each remaining cut-edge in $C_{rs}$ lies in a path of length at most $\kappa$ from a pair of oppositely labeled vertices (i.e., the end points of $e_1$). Therefore, if one were to do a bisection search on each path connecting these oppositely labeled vertices, one can find all the cut-edges using no more that $\left| C_{rs} \right| \log \kappa$ queries. If $\kappa$ is small, this quantity could be significantly smaller than $n$ (which is what an exhaustive querying algorithm would do). The reason for the drastically reduced query complexity is the fact that the cut-edges in $C_{rs}$ are \emph{tightly clustered}. A generalization of this observation gives rise to our third notion of complexity. 

Let $H_r = (C,\mathcal{E})$ be a ``meta graph'' whose vertices are the cut-edges of $G$ and $\left\{ e,e' \right\}\in \mathcal{E}$ iff $\delta(e,e')\leq r$, i.e., $H_r$ is the $r$-nearest neighbor graph of the cut-edges, where distances are defined according to $\delta$. From the definition of $\delta$, it is clear that for any $r\in \mathbb{N}$, $H_r$ has at least $m$ connected components. 
\begin{definition}
\label{def.clusteredness}
A cut-set $C$ is said to be $\kappa-${\bf clustered} if $H_\kappa$ has exactly $m$ connected components. These connected components correspond to the cut-components, and we will say that each individual cut-component is also $\kappa$-clustered.  
\end{definition}

Turning back to Figure~2, observe that the cut component $C_{12}$ is $\kappa$-clustered for any $\kappa\geq 5$ in Figure~\ref{fig.cutComponent}(a) and $\kappa-$clustered for any $\kappa \geq 12$ in Figure~\ref{fig.cutComponent}(b). The figure also shows a length 5 (resp. 12) path in Fig~\ref{fig.cutComponent}(a) (resp. Fig~\ref{fig.cutComponent}(b)) that defines the clusteredness.

\begin{figure}[h!]
\begin{center}
\includegraphics[width=0.65\textwidth]{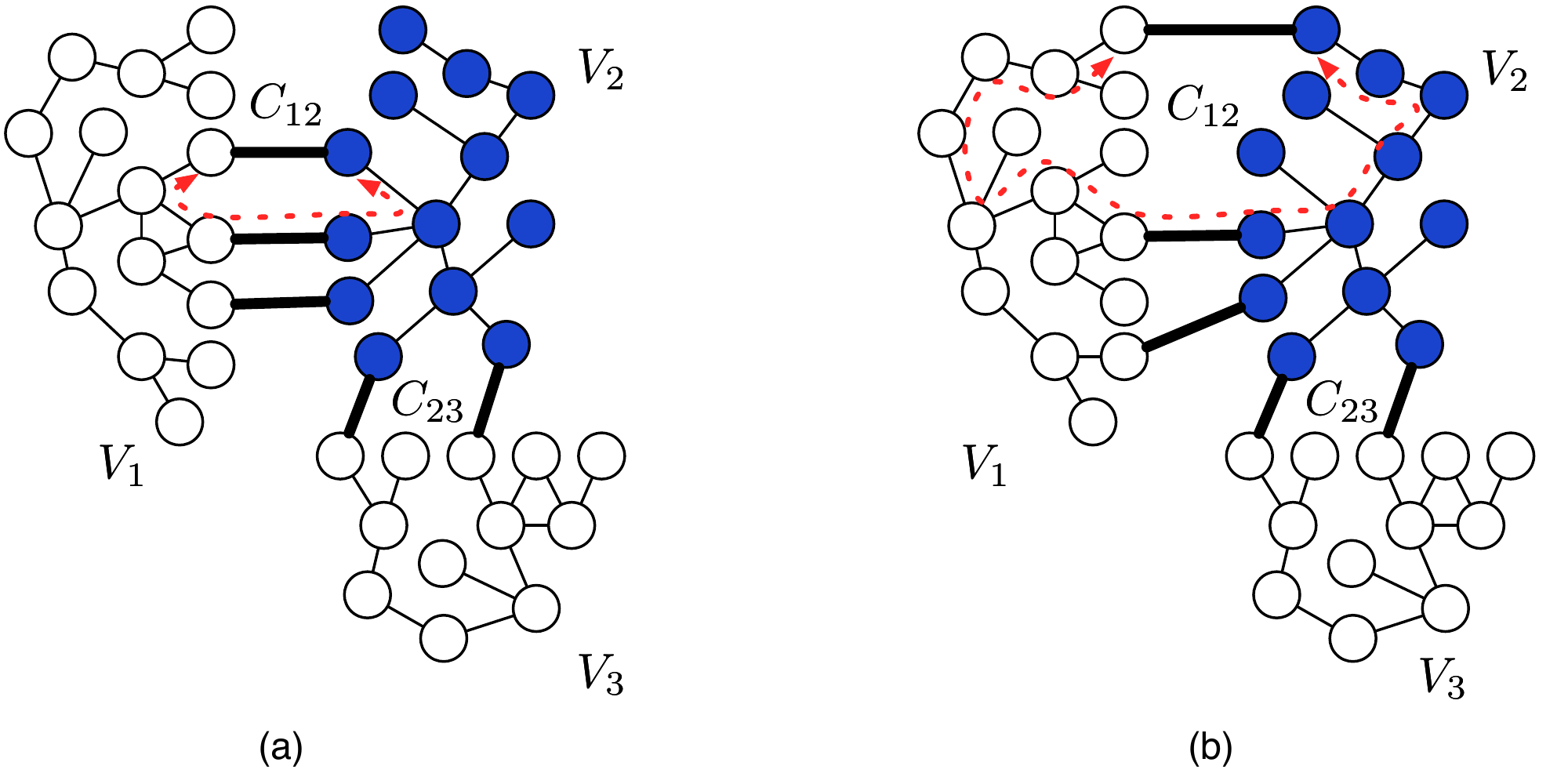}
\caption{Two graphs with the same number of cut-edges (thickened edges) but the cut component $C_{12}$ in (a) is more ``clustered'' than its counterpart in (b). $C_{12}$ is 5-clustered  in (a) and 12-clustered in (b). The corresponding length $\kappa$ paths are shows with a (red) dotted line.}
\label{fig.cutComponent}
\end{center}
\end{figure}

\subsection{The query complexity of \ssq}
In this section, we will prove the following theorem. 
\begin{theorem}
\label{thm.main}
Suppose that a graph ${G}=(V,E)$ and a binary function $f$ are such that the induced cut set $C$ has $m$ cut components that are each $\kappa-$clustered. Then for any $\epsilon>0$, \ssq will recover $C$ with probability at least $1-\epsilon$ if the \textsc{budget} is at least
\begin{equation}
\frac{\log(1/(\beta \epsilon))}{\log\left({1}/{(1-\beta)}\right)}+m\left(\lceil\log_2 n\rceil - \lceil \log_2 \kappa\rceil\right) + \card{\partial C}\left( \lceil \log_2 \kappa\rceil +1 \right)
\label{eq.mainBound}
\end{equation}
\end{theorem}
Before we prove the theorem, we will make some remarks: 

{\bf 1.} As observed in \cite{afshani07}, using $\mathcal{O}(\left| \partial C \right|)$ queries, we can create a greedy cover of the graph and reduce the length of the longest binary search from $n$ to $\mathcal{O}\left(\frac{n}{\left|\partial C\right|}\right)$. With this simple modification to \ssq, the $\log_2 n$ in \eqref{eq.mainBound} can be replaced by $\log(n/\left| \partial C \right|)$.

{\bf 2.} Suppose that $G$ is a $\sqrt{n}\times\sqrt{n}$ grid graph where one half of this graph, i.e., a $\sqrt{n}\times\sqrt{n}/2$ rectangular grid, is labeled $+1$ and the other half  $-1$. In this case, since $m=1, \left|\partial C\right|=\mathcal{O}(\sqrt{n})$ and $\kappa =3$, Theorem~\ref{thm.main} says that \ssq needs to query at most $\mathcal{O}(\sqrt{n} + \log n)$ vertices. This is much smaller than the bound of $\mathcal{O}\left(\sqrt{n}\log n\right)$ which can be obtained using the results of \cite{afshani07}.

In fact, when $\kappa< n/|\partial C|$, it is clear that for $m\geq 1$, $m \log(n/|\partial C|) + (|\partial C|-m) \log\kappa < |\partial C| \log(n/|\partial C|)$, i.e., our bounds are strictly better than those in \cite{afshani07} \footnote{the first term in \eqref{eq.mainBound}  is due to random sampling and is present in both results. We have also ignored integer effects in making this observation. }. 

{\bf 3.} As discussed in Section~\ref{sec.preliminaries}, the number of queries the (i.i.d) noise tolerant version of \ssq submits can be bounded by the above quantity times 
$\left\lceil \frac{1}{2\gamma^2}\log\left( \frac{n}{\epsilon} \right) \right\rceil$
 and is guaranteed to recover $C$ exactly with probability at least $1-2 \epsilon$.
 
{\bf 4.} The first term in \eqref{eq.mainBound} is due to random sampling phase which can be quite large if $\beta$ is very small. Intuitively, if $V_i$ is a very small connected component, then one might have to query almost all vertices before one can even obtain a label from $V_i$. Therefore, the ``balanced'' situation, where $\beta$ is a constant independent of $n$, is ideal here. These considerations come up in \cite{afshani07} as well, and in fact their algorithm needs a priori knowledge of $\beta$. Such balancedness assumptions arise in various other lines of work as well (e.g., \cite{edsn11ActiveClustering, BalakrishnanXuKrishSingh11Spectral}). Finally, we remark that if one is willing to ignore errors made on small ``outlier'' components, then $\beta$ can be redefined in terms of only the sufficiently large $V_i$'s. This allows us to readily generalize our results to the case where one can approximately recover the labeling of the graph. 

{\bf 5.} As we argue in Appendix~\ref{sec.tightness}, \ssq is near optimal with respect to the complexity  parametrization introduced in this paper. That is, we show that there exists a family of graphs such that \emph{no algorithm} has significantly lower query complexity than \ssq. It will be interesting to investigate the ``instance-optimality'' property of \ssq, where we fix a graph and then ask if \ssq is near optimal in discovering labelings on this graph. We take a step in this direction in Appendix~\ref{sec.2dGrid} and show that for the 2-dimensional grid graph, the number of queries made by \ssq is near optimal.

\begin{proof}{\bf of Theorem~\ref{thm.main}\; } 
We will begin the proof with a definition and an observation. Recall (from Section~\ref{sec.complexity}) that the labeling function partitions the vertex set into similarly labeled components $V_1,\ldots, V_k$. Suppose that $W\subset V$ is such that for each $i\in\left\{1,2,\ldots,k\right\}$, $\left|W\cap V_i\right|\geq 1$. Then, it follows that for each $e\in C$, there exists a pair of vertices $v,v'$ in $W$ such that $f(v)\neq f(v')$ and $e$ lies on a path joining $v$ and $v'$. We call such a set a \emph{witness} to the cutset $C$. Since $\partial C$ is a witness to the cut-set $C$, it is clearly necessary for any algorithm to know the labels of a witness to the cut set $C$ in order to learn the cut set. 

Now, as described in Section~\ref{sec.algorithm}, the operation of \ssq consists of two phases -- (a) the random sampling phase and (b) the aggressive search phase. Observe that if \ssq knows the labels of a witness to the cut component $C_{rs}$, then, until it discovers the entire cut-component, \ssq remains in the aggressive search phase. Therefore, the goal of the random sampling phase is to find a witness to the cut set. This implies that we can bound from above the number of random queries \ssq needs before it can locate a witness set, and this is done in the following lemma. 
\begin{lemma}
\label{lem.randomSampling}
Consider a graph $\mathcal{G}=(V,E)$ and a labeling function $f$ with balancedness $\beta$. For all $\alpha>0$, a subset $L$ chosen uniformly at random
is a witness to the cut-set with probability at least $1-\alpha$, as long as
\begin{equation*}
\card{L}\geq \frac{\log(1/(\beta \alpha))}{\log\left({1}/{(1-\beta)}\right)}.
\end{equation*}
\end{lemma}

We refer the reader to Appendix~\ref{sec.randomSamplingProof} for the proof which is a straightforward combinatorial argument. 

Of course, since the random sampling phase and the aggressive search phase are interleaved, \ssq might end up needing far fewer random queries before a witness $W$ to the cut-set has been identified. 

Now, we will turn our attention to the aggressive search phase. The algorithm enters the aggressive search phase as soon there are oppositely labeled vertices that are connected (recall that the algorithm removes any cut-edge that it finds). Let $\ell_{S^2}(G,L,f)$ be the length of the shortest among all shortest paths connecting oppositely labeled vertices in $L$, and we will suppress this quantity's dependence on $G, L, f$ when the context is clear. After each step of the aggressive search phase, the shortest shortest paths that connect oppositely labeled vertices gets bisected. Suppose that, during this phase, the current shortest path length is $\ell_{S^2} = \ell\in \mathbb{N}$. Then depending on its parity, after one step $\ell_{S^2}$ gets updated to $\frac{\ell + 1}{2}$ (if $\ell$ is odd) or $\frac{\ell}{2}$ (if $\ell$ is even). Therefore, we have the following result .
\begin{claim}
\label{claim.oneRun}
If $\ell_{S^2} = \ell$, after no more than $r = \lceil \log_2\ell \rceil + 1$ aggressive steps, a cut-edge is found. 
\end{claim}
\begin{proof}
First observe that after $r$ steps of the aggressive search phase, the length of the shortest shortest path is no more than $\frac{\ell + 2r-1}{2^r}$. Next, observe that if $r \geq 4$, then $2^{r-1} \leq 2^r - 2r +1$. 

The proof of this claim then follows from the following observation. Let us first suppose that $\ell \geq 8$, then setting $r = \lceil \log_2\ell \rceil + 1$, we have that $\ell \leq 2^{r-1} \leq 2^r - 2r +1$, which of course implies that $\frac{\ell+2r-1}{2^r}\leq 1$.
Therefore, after at most $r$ steps, the current shortest shortest path length drops to below 1, i.e., a cut-edge will be found. For the case when $\ell<8$, one can exhaustively check that this statement is true. 
\end{proof}
It is instructive to observe two things here: (a) even though $\ell_{S^2}$ gets (nearly) halved after every such aggressive step, it might not correspond to the length of a single path through the span of the above $r$ steps, and (b) at the end of $r$ aggressive steps as above, at least one new boundary vertex is found, therefore, the algorithm might end up uncovering multiple cut-edges after $r$ steps.

To bound the number of active queries needed, let us split up the aggressive search phase of \ssq into ``runs'', where each run ends when a new boundary vertex has been discovered, and commences either after the random sampling phase exposes a new path that connects oppositely labeled vertices or when the previous run ends in a new boundary vertex being discovered. Let $R$ be the total number of runs. Notice that $R\leq \left| \partial C \right|$. For each $i\in R$, let $G_i$ and $L_i$ denote the graph and the label set at the start of run $i$. Therefore, by Claim~\ref{claim.oneRun}, the total number of active learning queries needed can be upper bounded by $\sum_{i\in R}\left\{ \lceil \log_2 \left(\ell_{S^2}(G_i,L_i,f)\right) \rceil + 1\right\}$. 
\renewcommand{\thefootnote}{\fnsymbol{footnote}}

Now, observe that for each run in $i\in R$, it trivially holds that $\ell_{S^2}(G_i,L_i,f) \leq n$ \footnote{As in the Remark~1 after the statement of the theorem, this can be improved to $\mathcal{O}\left(n/\left| \partial C \right|\right)$ using a greedy cover of the graph \citep{afshani07}.}. Now suppose that one cut-edge is discovered in a cut-component $C_{rs}$.  Since the graph on $C_{rs}$ induced by $H_\kappa$ is $\kappa$-connected by the assumption of the theorem, until all the cut-edges are discovered, there exists at least one undiscovered cut-edge in $C_{rs}$ that is at most $\kappa$ away (in the sense of $\delta$) from one of the discovered cut-edges. Therefore, $\ell_{S^2}$ is no more than $\kappa$ until all cuts in $C_{rs}$ have been discovered. In other words, for $\left| C_{rs} \right|-1$ runs in $R$, $\ell_{S^2}\leq \kappa$. 

Reasoning similarly for each of the $m$ cut components, we have that there are no more than $m$ ``long'' runs (one per cut-component) and we will bound $\ell_{S^2}$ for these runs by $n$. From the argument above, after an edge from a cut-component has been found, $\ell_{S^2}$ is never more than $\kappa$ till all the edges of the cut-component are discovered. Therefore, we have the following upper bound on the number of aggressive queries needed.
\begin{align*}
\sum_{i\in R}\left\{ \lceil \log_2 \left(\ell_{S^2}(G_i,L_i,f)\right) \rceil + 1\right\} &\leq \left| \partial C \right| + m \lceil \log_2 n \rceil + \left( \left| \partial C \right| -m \right)\lceil \log_2 \kappa\rceil\\
& = m \left(\lceil \log_2 n \rceil -  \lceil\log_2\kappa\rceil\right) + \left| \partial C \right| \left(\lceil\log_2 \kappa\rceil + 1\right)
\end{align*}
Putting this together with the upper bound on the number of random queries needed, we get the desired result. 
\end{proof}

\section{\ssq for Nonparametric Learning}
Significant progress has been made in terms of characterizing the theoretical advantages of active learning in nonparametric settings.  For instance, minimax rates of convergence in active learning have been characterized under the so called {\em boundary fragment} assumption \citep{castro08,wang11,Minsker12}. This  model requires the Bayes decision boundary to have a functional form.  For example, if the feature space is $[0,1]^d$, then the boundary fragment model assumes that the Bayes decision boundary is defined by a curve of the form \mbox{$x_d = f(x_1,\dots,x_{d-1})$,} for some (smooth) function $f$. While such assumptions have proved useful for theoretical analysis, they are unrealistic in practice. Nonparametric active learning has also been analyzed in terms of abstract concepts such as bracketing and covering numbers \citep{hanneke11,koltchinskii10}, but it can be difficult to apply these tools in practice as well. The algorithm proposed in \cite{zhu03combining} offers a flexible approach to nonparametric active learning that appears to provide good results in practice, but comes with no theoretical performance guarantees.  A contribution of our paper is to fill the gap between the practical method of \cite{zhu03combining} and the theoretical work above.  The \ssq algorithm can adapt to nonparametric decision boundaries without the unrealistic boundary fragment assumption required by \cite{castro08}, for example.  We show that \ssq achieves the minimax rate of convergence for classification problems with decision boundaries in the so-called {\em box-counting} class, which is far less restrictive than the boundary fragment model.  To the best of our knowledge this is the first practical algorithm that is near minimax-optimal for this class of problems.

\subsection{Box-Counting Class}
Consider a binary classification problem on the feature space $[0,1]^d$, $d\geq 1$.  The box-counting class of decision boundaries generalizes the set of boundary fragments with Lipschitz regularity to sets with arbitrary orientations, piecewise smoothness, and multiple connected components. Thus, it is a more realistic assumption than boundary fragments for classification; see \cite{scott06} for more details. Let $w$ be an integer and let $P_w$ denote the regular partition of $[0,1]^d$ into hypercubes of side length $1/w$. Every classification rule can be specified by a set $B \subset [0,1]^d$ on which it predicts $+1$. Let $N_w(B)$ be the number of cells in $P_w$ that intersect the boundary of $B$, denoted by $\partial B$. For $c_1 > 0$, define the box-counting class $\mathcal{B}_{\mbox{\tiny BOX}}(c_1)$ as the collection of all sets $B$ such that $N_w(B) \leq c_1 w^{d-1}$ for all (sufficiently large) $w$.

\subsection{Problem Set-up}
Consider the active learning problem under the following assumptions.

\begin{enumerate}
\item[\bf A1] The Bayes optimal classifier $B_\ast$ resides in $\mathcal{B}_{\mbox{\tiny BOX}}(c_1)$. The corresponding boundary $\partial B_\ast$ divides $[0,1]^d$ into $k$ connected components\footnote{with respect to the usual topology on $\mathbb{R}^d$} (each labeled either $+1$ or $-1$) and each with volume at least $0<\beta<1$. Furthermore, the Hausdorff distance between any two components with the same label is at least $\Delta_1>0$.
\item[\bf A2] The marginal distribution of features $P(X=x)$ is uniform over $[0,1]^d$ (the results can be generalized to continuous distributions bounded away from zero).
\item[\bf A3]  The conditional distribution of the label at each feature is bounded away from $1/2$ by a positive margin; i.e., $\left|P(Y=1|X=x) -1/2\right|\geq \gamma > 0$ for all
$x\in [0,1]^d$.
\end{enumerate}

It can be checked that the set of distributions that satisfy A1-A3 contain the set of distributions BF$(1,1,C_1,0,\gamma)$ from \cite{castro08}.

Let $G$ denote the  regular square lattice on $[0,1]^d$ with $w$ vertices equispaced in each coordinate direction. Each vertex is associated with the center of a cell in the partition $P_w$ described above. Figure~\ref{fig.gridGraph} depicts a case where $d=2$, $w=15$ and $k=2$, where the subset of vertices contained in the set $B_\ast$ is indicated in red. The minimum distance $\Delta_1$ in A1 above ensures that, for $w$ sufficiently large, the lattice graph also consists of exactly $k$ connected components of vertices having the same label. \ssq can be used to solve the active learning problem by applying it to the lattice graph associated with the partition $P_w$. In this case, when \ssq requests a (noisy) label of a vertex of $G$, we randomly draw a feature from the cell corresponding to this vertex and return its label.  

\subsection*{ Near Minimax Optimality of \ssq. } The main result of this section shows that the excess risk of \ssq for the nonparametric active learning problem described above is minimax-optimal (up to a logarithmic factor).  Recall that the excess risk is the difference between the probability of error of the \ssq estimator and the Bayes error rate. The {\em active learning} minimax lower bound for excess risk is given by $n^{-1/(d-1)}$ \citep{castro08}, a significant improvement upon the {\em passive learning} bound of $n^{-1/d}$ \citep{scott06}.  Previous algorithms (nearly) achieving this rate required the boundary fragment assumption \cite{castro08}, and so \ssq is near-optimal for a much larger and more realistic range of problems. 
 \begin{theorem}
 \label{thm.nonparametric}
For any classification problem satisfying conditions A1-A3, there exists a constant $\mbox{$C(k,\beta, \Delta_1, \gamma) >0$}$ such that the excess risk achieved by \ssq with $n$ samples on this problem is no more than $C(\frac{\log n}{n})^{\frac{1}{d-1}},$ for $n$ large enough. 
\end{theorem}

To prove this theorem, we use Theorem~\ref{thm.main} to bound the number of samples required by \ssq to be certain (with high probability) of the classification rule everywhere but the cells that are intersected by the boundary $\partial B_\ast$. We then use the regularity assumptions we make on the distribution of the features and the complexity of the boundary to arrive at a bound on the excess risk of the \ssq algorithm. This allows us to estimate the excess risk as a function of the number of samples obtained. Refer to Appendix~\ref{sec.proofOfTheorem5} for the details of the proof.  

\begin{figure}[h!]
\begin{center}
\includegraphics[scale=0.55]{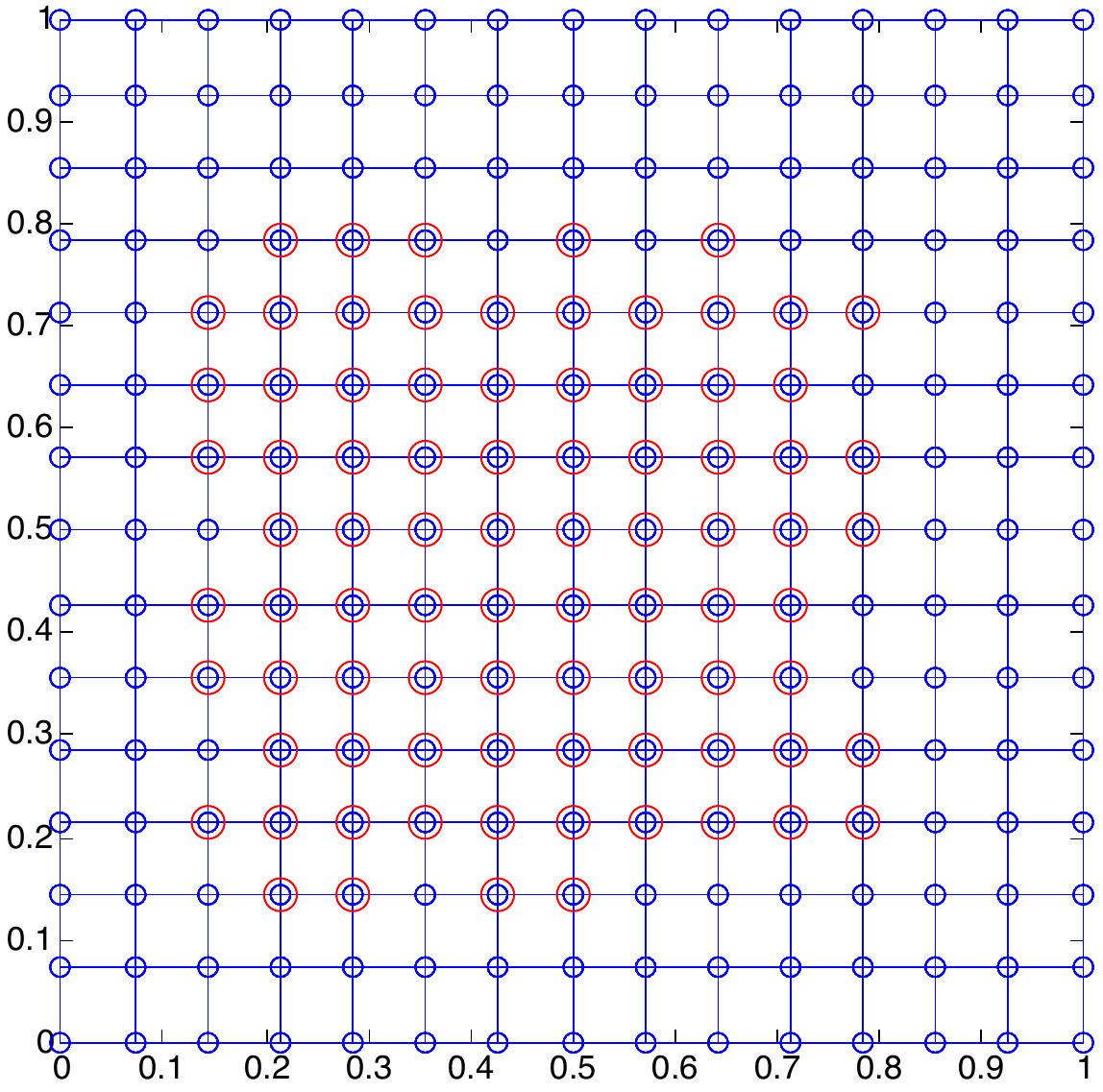}  
\captionof{figure}{15$\times$15 grid graph used in experiments. The vertices with the red circle indicate the positive class.}
\label{fig.gridGraph}
\end{center}
\end{figure}

\begin{table}
\begin{center}
\begin{tabular}{|l|@{\hspace{1mm}}l@{\hspace{2mm}}l@{\hspace{2mm}}l@{\hspace{2mm}}l@{\hspace{1mm}}|}
\hline
Graph : ($n$, $|C|$, $\card{\partial C}$) & \multicolumn{4}{l|}{\begin{tabular}{@{}l@{}}Mean $\partial C$-query \\ complexity (10 trials)\end{tabular}}  \\
& \ssq & AFS & ZLG & BND \\
\hline
Grid \;\;: (225, 32 ,68) & {\bf 88.8} & 160.4 & 91 & 192\\ 
1 v 2 \;\,: (400, 99, 92) & {\bf 96.4} & 223.2 & 102.6 & 370.2 \\
4 v 9 \;\,: (400, 457, 290) & {\bf 290.6} & 367.2  & 292.3 & 362.4\\
CVR\;\,: (380, 530, 234) & {\bf 235.8} & 332.1 & 236.2  &371.1\\ 
\hline
\end{tabular}
      \captionof{table}{Performance of \ssq, AFS, ZLG, BND.}
\end{center}
\end{table}

\section{Experiments}
\label{sec.experiments}
We performed some preliminary experiments on the following data sets: 

\noindent(a) {\bf Digits}: This dataset is from the Cedar Buffalo binary digits database originally~\cite{hull94database}. 
We preprocessed the digits by reducing the size of each image down to a  16x16 grid with down-sampling and Gaussian smoothing~\cite{lecun90handwritten}.  Each image is thus a 256-dimensional vector with elements being gray-scale pixel values in 0--255. We considered two separate binary classification tasks on this data set: 1 vs 2 and 4 vs 9. Intuitively one might expect the former task to be much simpler than the latter. For each task, we randomly chose 200 digits in the positive class and 200 in the negative. We computed the Euclidean distance between these 400 digits based on their  feature vectors. We then constructed a symmetrized 10-nearest-neighbor graph, with an edge between images $i,j$ iff $i$ is among $j$'s 10 nearest neighbors or vice versa.
Each task is thus represented by a graph with exactly 400 nodes and about 3000 undirected unweighted edges. Nonetheless, due to the intrinsic confusability, the cut size and the boundary (i.e., edges connecting the two classes) varies drastically across the tasks: 1 vs 2 has a boundary of 92, while 4 vs 9 has a boundary of 290. 

\noindent(b) {\bf Congressional Voting Records (CVR):} This is the congressional voting records data set from the UCI machine learning repository \citep{Bache+Lichman:2013}. We created a graph out of this by thresholding (at 0.5) the Euclidean distance between the data points. This was then processed to retain the largest connected component which had 380 vertices and a boundary size of 234.

\noindent(c) {\bf Grid: } This is a synthetic example of a 15x15 grid of vertices with a positive core in the center.  The core was generated from a square by randomly dithering its boundary. See Figure~\ref{fig.gridGraph}. 

We compared the performance of four algorithms: (a) $\mathbf{S^2}$ (b) {\bf AFS} -- the active learning algorithm from \cite{afshani07}; (c) {\bf ZLG} -- the algorithm from \cite{zhu03combining}; and (d)~{\bf BND} -- the experiment design-like algorithm from \cite{GuHan12bound}. 

We show the number of queries needed before all nodes in $\partial{C}$ have been queried.
This number, which we call $\partial C$-query complexity, is by definition no smaller than $\card{\partial C}$. 
Notice that before completely querying $\partial C$, it is impossible for any algorithm to guarantee zero error without prior assumptions.
Thus we posit that $\partial C$-query complexity is a sensible measure for the setting considered in this paper. In fact $\partial C$-query complexity can be thought of as the experimental analogue of the theoretical query complexity of Section~\ref{sec.analysis}. These results are shown in Table~1. The bold figures show the best performance in each experiment. As can be seen, \ssq clearly outperforms AFS and BOUND as suggested by our theory. It is quite surprising to see how well ZLG  performs given that it was not designed with this objective in mind. 
We believe that trying to understanding this will be a fruitful avenue for future work.

\bibliographystyle{plainnat}
\bibliography{dasarathy15}

\section*{Acknowledgements}
RN is supported in part by the National Science Foundation grant CCF‐1218189 and the National Institutes of Health grant 1 U54 AI117924-01; XZ is supported in part by National Science Foundation grants IIS 0916038, IIS 0953219, IIS 1216758, and the National Institutes of Health grant 1 U54 AI117924-01.

\appendix
\newtheorem*{prop1}{Proposition~\ref{prop.noise}}
\section{Proof of Proposition~\ref{prop.noise}}
\label{sec.proofProp1}
\begin{prop1}
Suppose $\mathcal{A}$ is an algorithm that has access to $f$ through a noiseless oracle (i.e., a $0$noisy oracle), and suppose that it has a $\epsilon-$query complexity of $q$, then for each $\gamma\in (0,0.5)$, there exists an algorithm $\tilde{\mathcal{A}}$ which, using a $\gamma-$noisy oracle achieves a $2 \epsilon-$query complexity given by $q\times \left\lceil \frac{1}{2(0.5 - \gamma)^2}\log\left( \frac{n}{\epsilon} \right) \right\rceil$.
\end{prop1}
\begin{proof}
Given a $\gamma>0$, one can design $\tilde{\mathcal{A}}$ as follows. $\tilde{\mathcal{A}}$ simply runs $\mathcal{A}$ as a sub-routine. Suppose $\mathcal{A}$ requires the label of a vertex $v\in V$ to proceed, $\tilde{\mathcal{A}}$ intercepts this label request and repeatedly queries the $\gamma-$noisy oracle $r$ (as defined above) times for the label of $v$. It then returns the majority vote $\tilde{f}(v)$ as the label of $v$ to $\mathcal{A}$. The probability that such an algorithm fails can be bounded as follows. 
\begin{align}
\mathbb{P}\left[ \tilde{\mathcal{A}} \mbox{ fails after $rq$ queries} \right]&\leq \mathbb{P}\left[ \mbox{$\exists v\in V$ s.t. $\tilde{f}(v) \neq f(v)$ }\right]\\&\qquad + \mathbb{P}\left[ \mathcal{A} \mbox{ fails after $q$ queries}\mid \tilde{f}(v) = f(v), \forall v\in V\right]\\
&\stackrel{(a)}{\leq} n\times  \mathbb{P}\left[ \tilde{f}(v) \neq f(v) \right] + \epsilon\\
&\stackrel{(b)}{\leq} n\times  e^{-2r (0.5 -\gamma)^2} + \epsilon \label{eq:prop1},
\end{align}
where $(a)$ follows from the union bound and fact that $\mathcal{A}$ has a $\epsilon-$query complexity of $q$. $(b)$ follows from applying the Chernoff bound \citep{chernoff1952measure} to the majority voting procedure :  $\mathbb{P}\left[ \tilde{f}(v) \neq f(v) \right] = \mathbb{P}\left[ {\rm Bin}(r, \gamma) \geq 0.5 \times r \right] \leq e^{-2r (0.5 - \gamma)^2}$. Therefore, if we set $r$ as in the statement of the proposition,  we get the desired result. 
\end{proof}
\section{Proof of Lemma~\ref{lem.randomSampling}}
\label{sec.randomSamplingProof}
\newtheorem*{randomSamplingLemma}{Lemma~\ref{lem.randomSampling}}
\begin{randomSamplingLemma}
Consider a graph $\mathcal{G}=(V,E)$ and a labeling function $f$ with balancedness $\beta$. For all $\alpha>0$, a subset $L$ chosen uniformly at random 
is a witness to the cut-set with probability at least $1-\alpha$, as long as
$\card{L}\geq \frac{\log(1/(\beta \alpha))}{\log\left({1}/{(1-\beta)}\right)}.$
\end{randomSamplingLemma}
\begin{proof}
The smallest component in $V$ is of size at least $\beta n$. Let $\mathcal{E}$ denote the event that there exists a component $V_i$ such that $V_i\cap L = \emptyset$ and let $\bar{\beta} = 1-\beta$. Then, using the union bound and ignoring integer effects, we have
$\mathbb{P}\left[\mathcal{E}\right]\leq\frac{1}{\beta}\cdot\frac{{\bar{\beta}n\choose \left| L \right|}}{{n\choose \left| L \right|}}< \frac{\bar{\beta}^{\left| L \right|}}{\beta},$ where the last inequality follows from the fact that $\bar{\beta}< 1$. 

To conclude, we observe that if we pick $\left| L \right|$ as stated in the lemma, then the right-hand side of the above equation drops below $\alpha$. This concludes the proof.
\end{proof}

\section{Proof of Theorem ~\ref{thm.nonparametric}}
\label{sec.proofOfTheorem5}
Recall that we propose to run \ssq on the lattice graph $G$ corresponding to the partition $P_w$ of $[0,1]^d$. And, recall that when \ssq requests a noisy sample of a vertex in $G$, a random feature is drawn from the cell that correpsonds to this vertex and its label is returned. Assumption A3 therefore implies that \ssq has access to a $\gamma-$noisy oracle\footnote{To be precise, \ssq has access to a $\gamma-$noisy oracle only when querying a vertex whose cell in $P_w$ does not intersect the boundary $\partial B_\ast$. However, in the sequel we will assume that \ssq fails to correctly predict the label of any vertex whose cell intersects the boundary, and therefore, this detail does not affect the analysis.}. 
In what follows, we will derive bounds on the $\epsilon-$query complexity of \ssq in this setting assuming it has access to a noiseless oracle. Then, arguing as in Proposition~\ref{prop.noise}, we can repeat each such query requested by \ssq a total of  $\frac{1}{2(0.5 - \gamma)^2}\log(w^d/\epsilon)$ times and take a majority vote in order to get the $2\epsilon-$query complexity for a $\gamma-$noisy oracle. So, in the sequel we will assume that we have access to a noise free oracle. 

We assume that $w$ is sufficiently large so that $w > \max\{\beta^{-1/d}, 2 \Delta_1^{-1}\}$. The first condition is needed since each homogenously labeled component of the problem corresponds to at least $\beta w^d$ vertices of the lattice graph $G$ and the second condition ensures that there are exactly $k$ connected components in $G$. 

First, we observe that by Lemma~\ref{lem.randomSampling}, if \ssq randomly queries at least $\log(1/\beta\epsilon)/\log(1/(1 - \beta))$ vertices, then with probability greater than $1-\epsilon$ it will discover at least one vertex in each of the $k$ connected components.   Next, we observe that since there are $k$ connected components of vertices, there are no more than $k^2/4$ cut components in the cut-set\footnote{Suppose there are $z_1$ components of label $+1$ and $z_2$ components of label $-1$, then there are at most $z_1 z_2$ cut components. The maximum value this can take when $z_1 + z_2 = k$ is $k^2/4$ by the arithmetic mean - geometric mean inequality.}. As described in the proof of Theorem~\ref{thm.main}, once it knows a witness to the cut set, the number of queries made by \ssq can be bounded by adding the number of queries it makes to first discover one cut-edge per each of these cut components to the number of queries needed to  perform local search in each of the cut components. Reasoning as in the proof of Theorem~\ref{thm.main}, for each cut component, \ssq requires no more than $\log w^d$ queries to find one cut edge. To bound the number of queries needed for the local search, we need to bound the size of the boundary $\left| \partial C \right|$ and $\kappa$, the clusteredness parameter (see Definition~\ref{def.clusteredness}) of the cut set in $G$. Towards this end, observe that by assumption A1, there are at most $c_1 w^{d-1}$ cells of the partition $P_w$ that intersects with $\partial B_\ast$. Since each cell has at most $d 2^{d-1}$ edges, we can conclude that $\left| \partial C \right|\leq 2c_1d (2w)^{d-1}$. To bound $\kappa$, let us fix a cut component  and note that given one cut edge, there must exist at least one other cut on a two dimensional face of the hypercube containing the first cut edge. Furthermore, this cut edge is contained on a path of length $3$ between the vertices of the first cut edge. Since the boundaries of the homogenously labeled components in $[0,1]^d$ are continuous (by definition), it is easy to see that there is an ordering of cut-edges in this cut component $e_1,e_2,\dots$ such that the distances between them, per Definition~\ref{def.clusteredness}, satisfy $\delta(e_i,e_{i+1})=3$. Since this holds true for all the cut components, this implies that the cut-set is 3-clustered, i.e, $\kappa=3$. Therefore, the complete boundary  can be determined by labeling no more than $(\lceil \log \kappa \rceil  + 1) \left| \partial C \right| = 6c_1 d(2w)^{d-1}$ vertices (since $\lceil \log 3\rceil=2$). As observed earlier, if we repeat each of the above queries $\frac{1}{2(0.5 - \gamma)^2} \log(w^d/\epsilon)$ times, we can see that if the  number of queries made by \ssq satisfies 
\begin{equation}
n \geq \left(6c_1 (2w)^{d-1} + \frac{k^2}{4} \log w^d + \frac{\log(1/\beta\epsilon)}{\log(1/(1 - \beta))}\right)\times \frac{\log(w^d/\epsilon)}{2(0.5 - \gamma)^2}\label{eq.sampleComplexityNonparametric},
\end{equation} 
one can ensure that with probability at least $1-2 \epsilon$, \ssq will only possibly make mistakes on the boundary of the Bayes optimal classifier $\partial B_\ast$.  Let $\mathcal{E}$ be this event and therefore $\mathbb{P}[\mathcal{E}^c] \leq 2 \epsilon$. 

If we let $S^2(X)$ and $B_\ast(X)$ denote the prediction of a feature $X$ by \ssq and the Bayes optimal classifier respectively, observe that the excess risk $ER[S^2]$ of \ssq satisfies 
\begin{align}
ER[S^2] & = \mathbb{P}\left[ S^2(X) \neq Y \right] - \mathbb{P}\left[ B_\ast(X) \neq Y \right]\\
&\stackrel{(a)}{\leq} \mathbb{P}[\mathcal{E}^c] + \mathbb{P}\left[ S^2(X) \neq Y \middle | \mathcal{E}\right] - \mathbb{P}\left[ B_\ast(X) \neq Y \middle | \mathcal{E}\right]\\
&\stackrel{(b)}{\leq} 2 \epsilon + \sum_{p\in P_w \cap \partial B_\ast}\mathbb{P}\left[ X\in  p \right]\\
&= \min\{2\epsilon + c_1 2^{d}w^{-1},1\},
\end{align}
where $(a)$ follows from conditioning on $\mathcal{E}$. $(b)$ follows from observing first  that conditioned on $\mathcal{E}$, $S^2(X)$ agrees with $B_\ast(X)$ everywhere except on the cells of $P_w$ that intersect the Bayes optimal boundary $\partial B_\ast$, and  that on this set, we can bound $\mathbb{P}\left[ S^2(X)\neq Y|\mathcal{E} \right] - \mathbb{P}\left[ B_\ast(X)\neq Y|\mathcal{E} \right] \leq 1$. The last step follows since by assumption A1, there are at most $2c_1 (2w)^{d-1}$ vertices on the boundary, and since, by assumption A2,  the probability of a randomly generated feature $X$ belonging to a specific boundary cell  is $w^{-d}$. Therefore, by taking $\epsilon = 1/w$ we have
that the excess risk of \ssq is bounded from above by  $(2 + c_12^d) w^{-1}$ and for this choice of $\epsilon$, the number of samples $n$  satisfies 
\begin{equation}
n \geq \left(6c_1 (2w)^{d-1} + \frac{k^2}{4} \log w^d + \frac{\log(w/\beta)}{\log(1/(1 - \beta))}\right)\times \frac{\log(w^{d+1})}{2(0.5 -\gamma)^2}\label{eq.sampleComplexityNonparametric2}.
\end{equation} 
Finaly, we conclude the proof by observing that if $n$ satisfies \eqref{eq.sampleComplexityNonparametric2} and if $w$ is sufficiently large, there exists a constant $C$ which depends only on $c_1, k, \beta, \gamma, d$ such that $C \left(\frac{\log n}{n}\right)^{\frac{1}{d-1}} \geq (2+c_12^d) w^{-1}$.

\section{The Tightness of \ssq}
\label{sec.tightness}
We will now argue that the upper bounds we derived in the main paper are tight. Towards this end, in what follows, we will assume that a witness (cf. Section~\ref{sec.analysis}) to the cut set is known. This allows us to separate the effect of the random sampling phase and the learning phase in such problems. 
\subsection{Parameter optimality of \ssq}
In this section, we will show that \ssq is near optimal with respect to the complexity parametrization introduced in Section~\ref{sec.complexity}. In particular, we will show that given particular values for $n, \kappa, m$ and  $\left| \partial C \right|$, there exists a graph such that no algorithm can significantly improve upon \ssq. For what follows, we will set $c \triangleq \left| \partial C \right|$. Let us define $$\mathcal{P} \triangleq \left\{ \left( n,c,m,\kappa \right) \in \mathbb{N}^4 : m\leq c, c(\kappa+2)\leq n \right\}.$$We will show that as long the given problem parameters are in $\mathcal{P}$, \ssq is near optimal. While it is trivially true that the number of cut components, $m$ has to be no more than $c$, the second condition places finer restrictions on the parameter values. These conditions are specific to the construction we present below and can be significantly weakened at the expense of clarity of presentation. 

\begin{figure}[h!]
\centering
  \includegraphics[scale = 0.45]{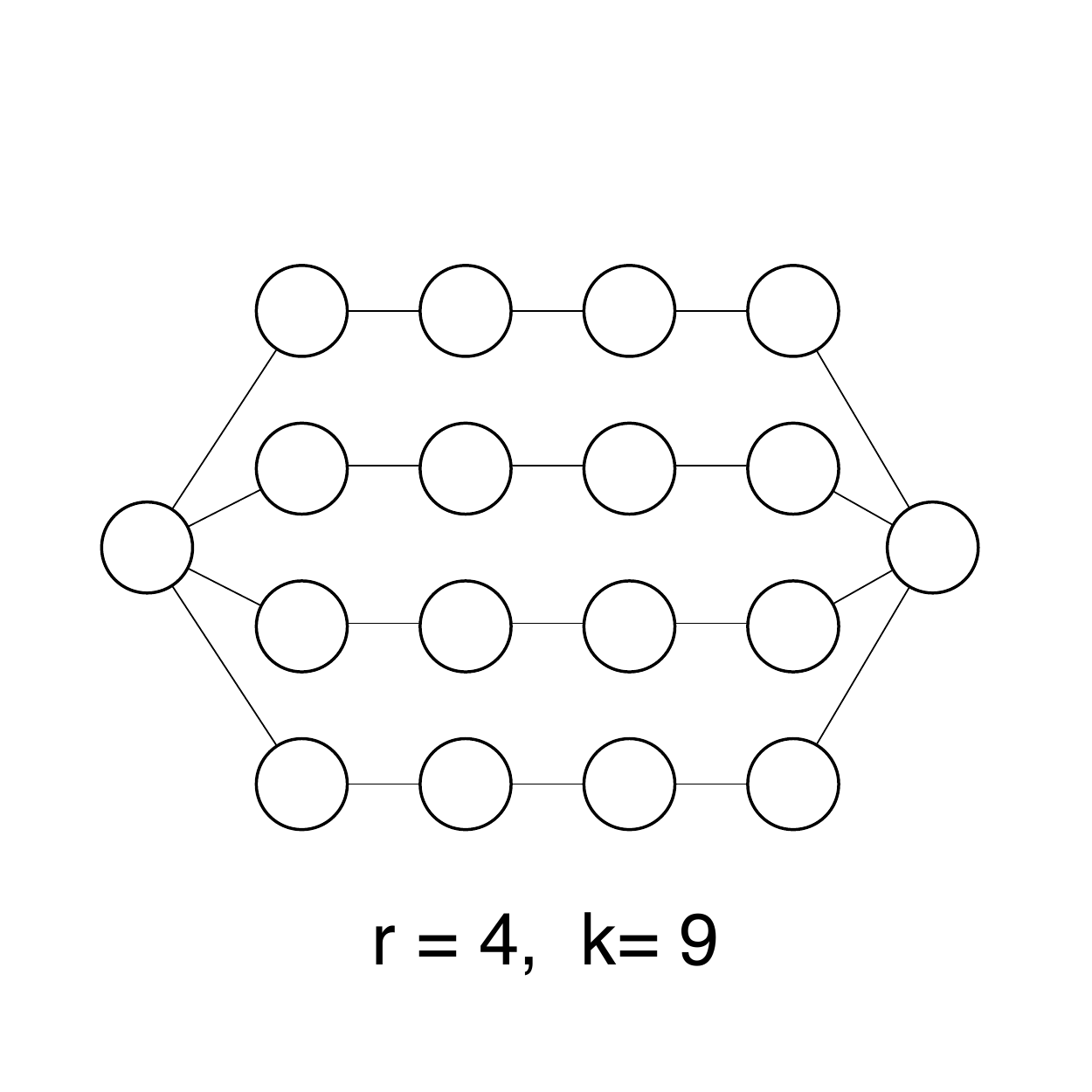}
  \caption{$G(4,9)$}
  \label{fig.tightness}
\end{figure}
\begin{figure}[h!]
\centering
  \includegraphics[scale = 0.5]{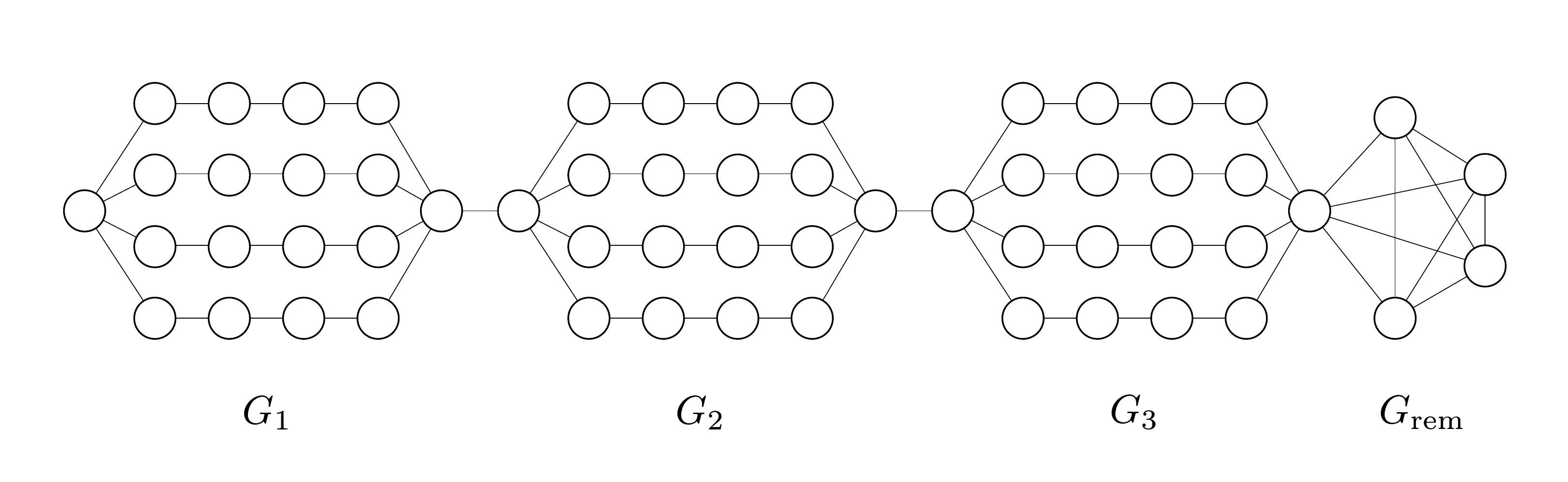}
  \caption{$p=3$ copies of $G(4,9)$ linked together}
  \label{fig.tightness2}
\end{figure}

We will now prove the following theorem. 
\begin{theorem}
Given a set of values for $n,c,m,\kappa \in \mathcal{P}$, there exists a graph $G$ on $n$ vertices and a set of labelings $\mathcal{F}$ on these vertices such that each $f\in \mathcal{F}$ satisfies: 
\begin{itemize}
\item $f$ induces \emph{no more than} $c$ cuts in the graph. 
\item The number of cut-components is $m$
\item Each component is $\kappa$-clustered.
\end{itemize}
Furthermore, $\log \left|\mathcal{F}\right|$ is no smaller than 
\begin{align*}
m &\log \left(\frac{1}{m}\left\lfloor \frac{n}{\left\lfloor \frac{c}{m} \right\rfloor (\left\lfloor \frac{\kappa - 1}{2} \right\rfloor) + 2} \right\rfloor\times \left(\left\lfloor \frac{\kappa - 1}{2} \right\rfloor + 1\right)\right)\\
&\qquad + \left(m \left\lfloor \frac{c}{m} \right\rfloor - m\right) \log\left(\left\lfloor \frac{\kappa-1}{2}\right\rfloor + 1\right).
\end{align*}
\end{theorem}
{\bf Remark: }Notice that $\log \left| \mathcal{F} \right|$ is a lower bound on the query complexity of any algorithm that endeavors to learn $c$ cuts in this graph that are decomposed into $m$ components, each of which is $\kappa-$clustered, \emph{even if} the algorithm knows the values $m, c,$ and $\kappa$. 

Now, this result tells us that if we assume, for the sake of simplicity, that $m$ evenly divides $c$, $\kappa$ is odd, and $c(\kappa + 1)$ evenly divides $2nm$ (notice that $c(\kappa+2)\leq n$ by assumption)  then, we have that 
\begin{align*}
\log \left|\mathcal{F}\right| &\geq m \log\left(\frac{1}{m} \left\lfloor \frac{2nm}{c(\kappa+1)} \right\rfloor\times \frac{\kappa+1}{2}\right) + \left(c - m\right)\log \left(\frac{\kappa+1}{2}\right)\\
& = m \log\left(\frac{n}{c}\right) + (c-m)\log \left(\frac{\kappa+1}{2}\right).
\end{align*}
Comparing this with Theorem~1 in the manuscript, we see that \ssq is indeed parameter optimal. 
\begin{proof}
First, define 
\begin{align*}
r &\triangleq \left\lfloor \frac{c}{m}\right\rfloor,\\
k &\triangleq 2\left\lfloor\frac{\kappa-1}{2}\right\rfloor + 1,\\
p &\triangleq \left\lfloor\frac{2n}{r \left(k-1\right)+4}\right\rfloor. 
\end{align*}
If $(n,c,m,\kappa)\in \mathcal{P}$, it can be shown that $r \geq 1$ and $p \geq m$. 
Let $G(r,k)$ denote the following graph on $\frac{r(k-1)+4}{2}$ vertices -- two vertices are connected by $r$ edge disjoint paths and each path has $\frac{k-1}{2}$ vertices (and $\frac{k+1}{2}$ edges). This is shown in Fig~4 for $r = 4$ and $k = 9$. $G$ is constructed by linking $p$ copies of $G(r,\kappa)$ and a ``remainder'' graph $G_{\rm rem}$ which is a clique with $n - p\left(r \left(\frac{k-1}{2}\right) + 2\right)$ as shown in Fig~5. We will  denote these $p$ copies of $G(r,k)$ as $G_1,\ldots, G_p$.

Let $\mathcal{F}$ be the set of all labelings obtained as follows:
\begin{enumerate}
\item Choose $m$ out of the $p$ subgraphs $G_1,\ldots,G_p$ without replacement. There are ${p \choose m}$ ways to do this. 
\item In each of these $m$ chosen subgraphs, pick $r$ edges to be cut edges, one on each of the $r$ paths. There are $\left(\frac{k+1}{2}\right)^r$ ways to do this in one subgraph. Therefore, there is a total number of $\left(\frac{k+1}{2}\right)^{mr}$ ways to do this.
\item Now, let the left most vertex of $G_1$ be labeled $+1$, the rest of the labels are completely determined by the cut-set we just chose. 
\end{enumerate}
Notice that for each $f\in \mathcal{F}$, the following holds: (a) there are exactly $m$ cut-components in the graph, (b) the number of cuts is $m\times \left\lfloor \frac{c}{m}\right\rfloor \leq c$, and (c) in each cut component, the cuts are $k\leq \kappa$ close. 

The total number of such labelings is given by: $${p\choose m} \times \left(\frac{k+1}{2}\right)^{mr} \geq \left(\frac{p}{m}\right)^m\left(\frac{k+1}{2}\right)^{rm}.$$
Therefore, we can lower bound $\log \left|\mathcal{F}\right|$ as follows
\begin{align*}
\log\left|\mathcal{F}\right| &\geq m \log \left(\frac{p}{m}\right) + mr \log \left(\frac{k+1}{2}\right)\\
& = m \log \left(\frac{1}{m}\left\lfloor \frac{2n}{r(k-1) + 4} \right\rfloor\right) + m \left\lfloor \frac{c}{m} \right\rfloor \log\left(\left\lfloor \frac{\kappa-1}{2}\right\rfloor + 1\right)\\
& = m \log \left(\frac{1}{m}\left\lfloor \frac{2n}{\left\lfloor \frac{c}{m} \right\rfloor 2(\left\lfloor \frac{\kappa - 1}{2} \right\rfloor) + 4} \right\rfloor\right) +  m \left\lfloor \frac{c}{m} \right\rfloor \log\left(\left\lfloor \frac{\kappa-1}{2}\right\rfloor + 1\right)\\
& = m \log \left(\frac{1}{m}\left\lfloor \frac{n}{\left\lfloor \frac{c}{m} \right\rfloor (\left\lfloor \frac{\kappa - 1}{2} \right\rfloor) + 2} \right\rfloor \left(\left\lfloor \frac{\kappa - 1}{2} \right\rfloor + 1\right)\right) \\&\qquad\qquad + \left(m \left\lfloor \frac{c}{m} \right\rfloor - m\right) \log\left(\left\lfloor \frac{\kappa-1}{2}\right\rfloor + 1\right) .
\end{align*}
This concludes the proof.
\end{proof}

\subsection{Two Dimensional Grid}
\label{sec.2dGrid}
In this section, we will show that in the case of the 2-dimensional grid \ssq is near optimal even if we fix the graph before hand. Notice that in this sense, this result is stronger than the one presented in the previous section and is particularly relevant to the theory of nonparametric active learning. Consider the example of a 2-dimensional $r\times r$ grid, where the bottom-left vertex is labeled $+1$ and the top-right vertex is labeled $-1$ (see Fig.~\ref{fig.tightness2}). We want to count/lower bound the total number of labelings such that there are exactly 2 connected components and such that the cut-size of the labeling is no more than $r$. Notice that the logarithm of this number will be a lower bound on the query complexity of any algorithm that endeavors to learn a cut set of size no more than $r$. 
\begin{theorem}
The number of ways to partition an $r\times r$ grid into 2 components using a cut of size at most $r$ such that the bottom-left vertex and top-right vertex are separated is lower bounded by $2^r$. 
\end{theorem}
\begin{figure}[h!]
\centering
  \includegraphics[scale = 0.75]{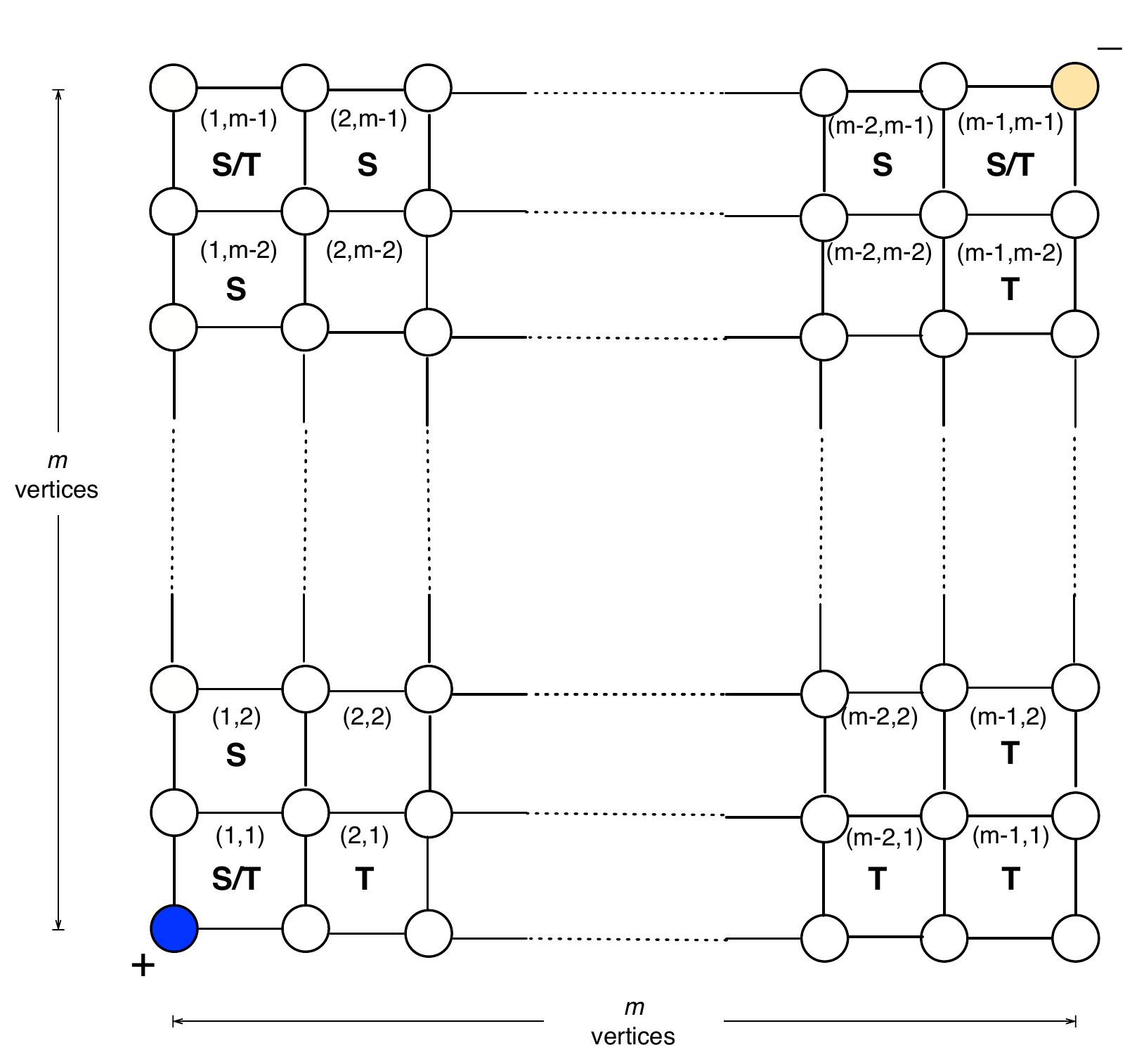}
  \caption{$r\times r$ grid with bottom-left and top-right vertex  oppositely labeled.}
  \label{fig.tightness3}
\end{figure}
\begin{proof}
We will first restrict our attention to cuts of size exactly $r$. Consider the grid shown in Figure~\ref{fig.tightness3}. We will begin by making a few simple observations. Each of the $(r-1)\times (r-1)$ boxes that contains a cut, has to contain at least 2 cuts. Furthermore, since these cuts are only allowed to partition the grid into 2 connected components, cuts are located in contiguous boxes. Therefore, there are at most $r-1$ boxes that contain cuts. 

We will think of a cut as a \emph{walk} on the grid of $(r-1)\times (m-1)$ boxes (labeled as $(i,j), 1\leq i,j\leq r-1$ in Fig~\ref{fig.tightness3}) and lower bound the total number of such walks. Observe that for a walk to correspond to a valid cut, it must contain one of the boxes labeled $S$ and one of the boxes labeled $T$. By symmetry, it suffices only consider walks that originate in an $S$ box and end in a $T$ box. 

To lower bound the number of valid walks, we are going to restrict ourselves to \emph{positive walks} -- walks that only move either right (R) or down (D). Notice that such walks traverse exactly $r-1$ boxes. Towards this end, we first observe that there are $2(r-1)$ such walks each of which originates in an $S$-block and terminates at the diametrically opposite $T$-block. These walks are made up of entirely R moves or entirely of D moves. Therefore, to count the number of remaining walks, we can restrict our attention to walks that contain at least one R move and one D move. Notice that such walks cannot cross over the diagonal. Therefore, by symmetry, it suffices to consider walks that start in an $S$-box on the left column: $\left\{(1,1),\ldots,(1,r-1)\right\}$ and end in a $T$-box in the bottom row: $\left\{(1,1),(2,1),\ldots,(r-1,1)\right\}$. Suppose, for $j\geq 2$, we start a walk at block $(1,j)$, then the walk has to make exactly $j-1$ down moves and $(m-2 - j+1)$ right moves (since the total number of blocks in the walk is $r-1$). Therefore, the total number of such positive walks that originate from $(1,j)$ is ${m-2 \choose j-1}$. Reasoning similarly, we conclude that there are $\sum_{j=2}^{r-3}{r-2\choose j-1}$ such positive walks from one of the $S$-boxes in the left column to one of the $T$-boxes in the bottom row.  Finally, observe that the walk that starts at $(1,1)$ and ends at $(1,m-1)$ correspond to two different cuts since the $(1,r-1)$ box has two valid edges that can be cut. Similarly the walk $(1,r-1)-\cdots-(r-1,r-1)$ corresponds to 2 valid cuts. Therefore, the total number of cuts from such walks is given by $2(r-1) + 2(2 + \sum_{j=2}^{r-3}{r-2\choose j-1}) = 2(r-1) + 2^{r-1}$, where the multiplication by $2$ inside follows from the symmetry we used. 

Observe now that if we allow the cuts to be smaller than $r$, then the total number of cuts is given by summing over the cut-sizes as follows: $\sum_{i=2}^r 2(i-1) + 2^{i-1} = {r-1 \choose 2} + 2^r - 2 \geq 2^r$. This concludes the proof.  
\end{proof}
Therefore, any algorithm will need to submit at least $\log(2^r) = \mathcal{O}(r)$ queries before it can discover a cut of size at most $r$ and in fact, from the proof above, this seems like a pretty weak lower bound (since we restricted ourselves only to positive walks). However, observe that since $\kappa = 3$ and $\left| \partial C \right| \leq r$ here, Theorem~1 (from the manuscript) tells us that \ssq submits no more than $\mathcal{O}\left( r  \right)$ queries for the same. Extending this argument to other families of graphs is an interesting avenue for future work. 

\end{document}